\documentclass[letterpaper, 10 pt, conference]{ieeeconf}  

\IEEEoverridecommandlockouts                              

\usepackage{amsmath}
\usepackage{amssymb}
\usepackage{amsthm}
\usepackage[thinc]{esdiff}
\usepackage{subcaption}
\usepackage{cleveref}
\usepackage{mathtools}
\usepackage{graphicx}
\usepackage{epstopdf}

\newcommand\norm[1]{\left\lVert#1\right\rVert}

\newcommand\R{\mathbb{R}}

\newcommand\A{\mathcal{A}}

\DeclareMathOperator{\pow}{pow}

\newtheoremstyle{style}
  {}
  {}
  {\itshape}
  {}
  {\bfseries}
  {.}
  { }
  {}
\theoremstyle{style}
\newtheorem{definition}{Definition}
\newtheorem{problem}{Problem}
\newtheorem{proposition}{Proposition}

\usepackage{algorithm}
\usepackage{algpseudocode}

\def\hb{{\bar{h}}}
\def\fb{{\bar{f}}}

\def\ifs{{\mathcal{I}}}
\def\is{{\iota}}

\def\M{{\mathcal{M}}}

\newcommand{\suppar}[1]{^{(#1)}}

\title{\LARGE \bf
Optimal Control of Sensor-Induced Illusions on Robotic Agents
}

\author{Lorenzo Medici$^1$, Steven M. LaValle$^1$ and Ba\c{s}ak Sak\c{c}ak$^{1,2}$
\thanks{This work was supported by a European Research Council Advanced Grant (ERC AdG, ILLUSIVE: Foundations of Perception Engineering, 101020977), Academy of Finland (projects BANG! 363637, CHiMP 342556).}
\thanks{$^{1}$Faculty of Information Technology and Electrical Engineering, University of Oulu, 
{\tt\small firstname.lastname@oulu.fi}}%
\thanks{$^{2}$Department of Advanced Computing Sciences, Maastricht University {\tt\small basak.sakcak@maastrichtuniversity.nl}}
}

\begin{document}

\maketitle
\thispagestyle{empty}
\pagestyle{empty}

\begin{abstract}
This paper presents a novel problem of creating and regulating localization and navigation illusions considering two agents: a receiver and a producer. A receiver is moving on a plane localizing itself using the intensity of signals from three known towers observed at its position. Based on this position estimate, it follows a simple policy to reach its goal. The key idea is that a producer alters the signal intensities to alter the position estimate of the receiver while ensuring it reaches a different destination with the belief that it reached its goal. We provide a precise mathematical formulation of this problem  and show that it allows standard techniques from control theory to be applied to generate localization and navigation illusions that result in a desired receiver behavior. 
\end{abstract}

\section{INTRODUCTION}
Humans experience illusions when sensory organs are exposed to stimuli that, when processed by the nervous system, result in a perceptual experience that does not coincide with the reality.
For example, the famous Ponzo illusion results from two identical lines being perceived as having different length \cite{ponzo1928urteilstauschungen,yildiz2022review}. 
Analogous to the human sensory organs and nervous system are the sensors and the processor of a robotic agent. The sensors provide measurements of the agent's body and the physical world, and a processor
condenses these measurements into an internal state that represents the state of the physical world. Therefore, robotic agents are susceptible to illusions similar to how humans are.

In this work, we consider robotic agents that can sense and act. We are interested in controlling the perception of a robotic agent, that is, controlling its internal state.  
Typically, it is not possible to alter the internal state directly, therefore, appropriate sensor measurements must be generated by altering the physical world that lead to a desired internal state. The agent determines its actions based on its internal state. Therefore, by altering its internal state, a desired behavior can be achieved. An important issue is to ensure that the generated perception is plausible, that is, it is consistent with the agent's model of the physical world. 

\begin{figure}
    \centering
\begin{subfigure}{0.52\linewidth}
            \includegraphics[width=.85\linewidth]{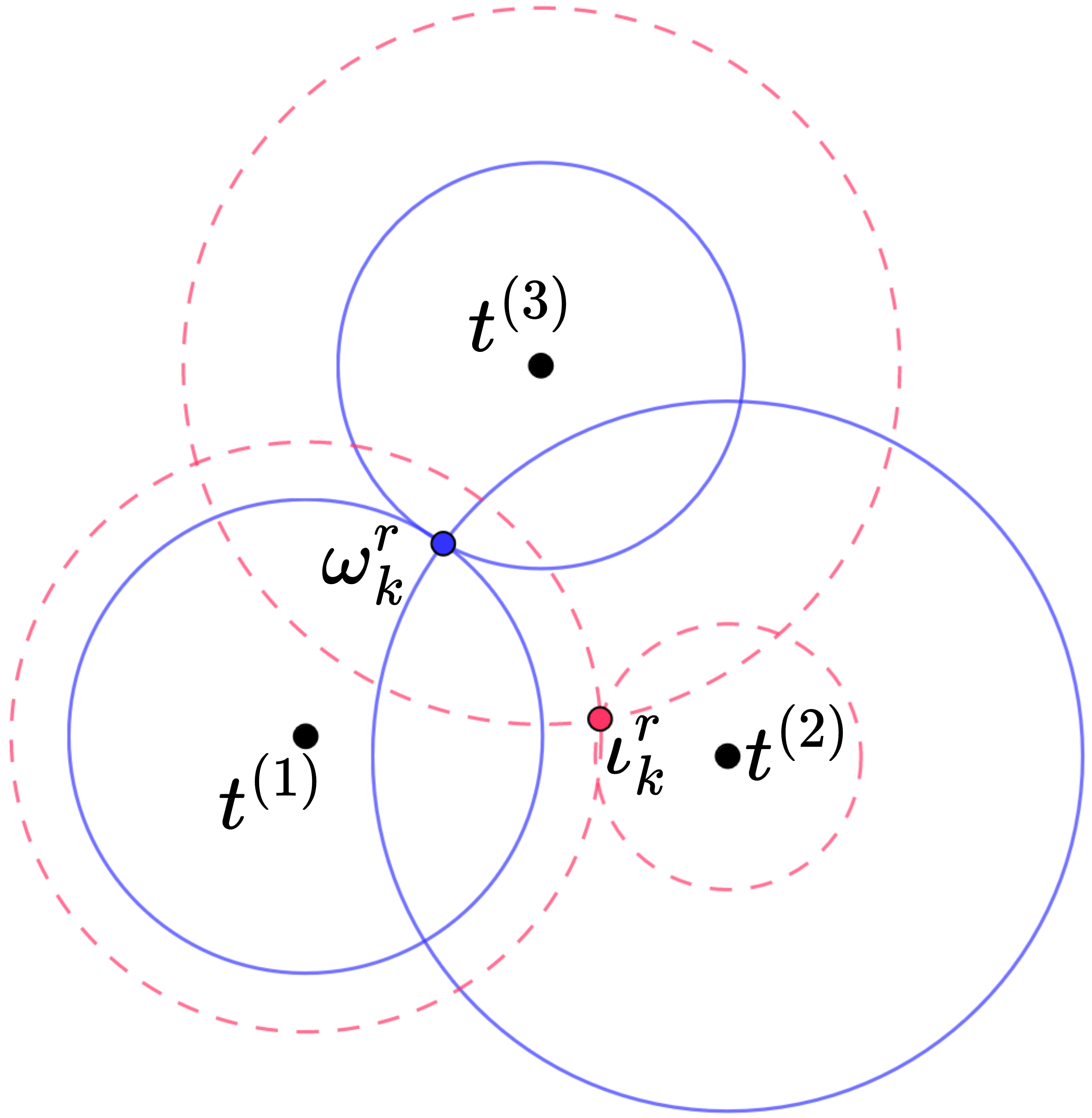}
        \caption{}
        \label{fig:tower-spoofing-figure}
    \end{subfigure}
\begin{subfigure}{0.45\linewidth}
    \includegraphics[width=.99\linewidth, trim={0cm 0cm 0cm 0.5cm},clip]{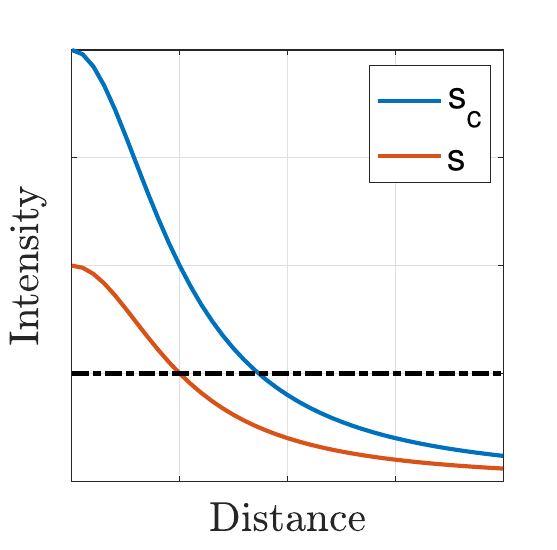}
    \caption{}
    \label{fig:signal-dropoff-models}
\end{subfigure}

    \caption{Localization illusions. (a) A receiver is placed at $\omega^r_k$ and the signal intensity is set to $s^{(i)}$ at the source for towers at $t^{(i)}$ for $i=1,2,3$. Considering the model of signal transmission with the signal intensity at source as $s_c$, the measured intensity at receiver position results in erroneous position estimate (intersection of red dashed circles). (b) Models of signal intensity as a function of distance from tower position (source). Blue line is the receiver model. Changing the signal intensity to different values leads to the perceived distance to the tower be different than the actual distance (see the intersection of the black line with the actual model and the receiver model).}
    \label{fig:changed-transmission-powers}
\vspace{-1.5em}
\end{figure}

Many of the ideas in this work will build upon \cite{LaValle-Big-PE}, in which a general mathematical formulation was introduced to model the process of altering the perceptual experience of an agent. In this work, we show how these ideas can be applied to pose and solve robotics problems.
The concrete problem that we address in this paper is creating localization and navigation illusions for an agent, a receiver, moving on a plane and localizing itself by trilateration using the intensity of signals emitted from three towers measured at agent position (see Fig.~\ref{fig:changed-transmission-powers}). The agent sensor model is based on inverse-square law parametrized with signal intensity at tower position used for calibration.
Based on this model, the agent can uniquely determine its position as the intersection of the level circles corresponding to measured intensities. However, setting the signal intensity at tower base different than the calibrated value results in a position estimate that is different than the agent's actual position. 
We consider a receiver agent with this type of localization following a state-feedback policy that leads the agent to its goal under correct localization. We address the problem of designing an agent, a producer, which controls the signal intensities of the towers. The goal of the producer is to ensure that the receiver reaches a goal that is different than its own while believing that it has reached its destination. 
We model the receiver dynamics together with its internal state, that is, the position estimate and design a regulator (producer) that stabilizes this system while considering plausibility constraints.

\subsection{Related work}
Generating stimuli to alter the perception of agents (biological or robotic) has received a sustained interest from different fields with applications ranging from designing virtual or augmented reality systems for training \cite{krosl2019icthroughvr} to studying animal behavior \cite{takalo2012fast}. 
Considering robotic agents, 
\cite{Suomalainen-virtual-reality-robots} has introduced the notion of \emph{virtual reality for robots} (VRR) presenting a general and formal way of describing virtual reality notions such as displays, virtual worlds, and rendering for robotics. Their notion of VRR encompasses physics-based simulators and sensor spoofing. 

Sensor spoofing literature, which deals with the problem of generating or mitigating adversarial attacks to a sensing system, is closely related to this paper. 
It has been shown that most sensors used in designing autonomous systems such as Inertial Measurement Units (IMUs)~\cite{son2015rocking,trippel2017walnut}, Light Detection and Ranging (LiDAR)~\cite{shin2017illusion} and Radar~\cite{Komissarov-radar-spoofing-attack, Kapoor-radar-spoofing-detection} sensors, and Global Navigation Satellite Systems (GNSS)~\cite{Kerns-GPS-spoofing, Dasgupta-GNSS-spoofing-detection} are susceptible to adversarial attacks which either render the sensor unavailable or purposefully generate false measurements that cause the system to deviate from its intended functioning. Similar to our work, in \cite{Kerns-GPS-spoofing}, the authors consider designing a \emph{spoofer} which alters the GNSS readings of a hostile Unmanned Aerial Vehicle (UAV) causing unrecoverable navigation errors and leading it to crash. Other than designing methods to spoof sensors, there is also an extensive literature on how to mitigate such attacks which deals with determining the plausibility of a sensor reading and designing robust control techniques with respect to potential attacks \cite{davidson2016controlling,jafarnia2012gps}. In this work, when designing a producer we will take plausibility of measurements with respect to receiver model into account but we will assume that the receiver is not equipped with anti-spoofing mechanisms. 

Designing simulators for testing the behavior of a robotic system is related as well. Similar to sensor spoofing, the objective of a simulator is also to generate stimuli. However, in the latter case the stimuli is expected to sufficiently mirror the expected real world scenarios without the agent actually being deployed in the real world.
The literature on simulation is vast and range from designing simulators to test the functioning of hardware \cite{bacic2005hardware} to test control and optimization software \cite{todorov2012mujoco,craighead2007survey}.
An interesting instance appears in \cite{Shell-reality-simulation-reality}, in which Shell and O'Kane have addressed the limits on curating a physical world given infrastructure constraints for simulation purposes. 

\section{Preliminary Definitions}
In this section, we briefly review the relevant background and define notions central to the paper.

\subsection{Agents in a universe} 

We model an agent capable of sensing and acting in an environment shared with other agents as two coupled systems named \emph{internal} and \emph{external} systems, similar to previous work \cite{sakcak2024mathematical}. An internal system is seen as a centralized computation component that processes sensor observations and actions. Consequently, an external system is the physical world. Although our notion of an internal system is general, for the scope of this paper, we will consider an internal system which maintains a representation of the external system state. 
To this end, we assume that the internal system of the agent is \emph{model-based} coinciding with the use of the term in robotics and control-theory. We call this model an \emph{intrinsic model} of the external system since it describes the external system from the agent's perspective. 

\begin{definition}[Intrinsic model] An intrinsic model of the external system corresponding to the sets of actions $U$ and observations $Y$ is a tuple $\M_{U\rightarrow Y} = (X, f, h)$ consisting of 
\begin{itemize}
    \item Set of states $X$. 
    \item A nondeterministic X-state transition function (XTF) $f : X \times U \rightarrow \pow(X) \setminus \emptyset$, in which, $\pow(X)$ denotes the power set of $x$, such that $f(x_k, u_k)$ is the set of possible states that can be reached at stage $k+1$ when an action $u_k$ is taken at state $x_k$ at stage $k$. 
    \item A sensor-mapping $h : X \rightarrow \pow(Y)$ such that $y_k \in h(x_k)$ is an observation that an agent receives when at state $x_k$ at stage $k$.
\end{itemize}
\end{definition}

We use $U \rightarrow Y$ in the subscript of $\M_{U\rightarrow Y}$ to emphasize that the model captures the input-output relationship between the actions applied to and observations received from the external system. 
In case the images of $f$ and $h$ are always singletons, $\M_{U\rightarrow Y}$ is a deterministic model of the external system. In this paper, we will only consider nondeterministic models together with deterministic models as a special case.
However, probabilistic formulations are also possible, see for example \cite{LaValle-Big-PE,sakcak2024mathematical}.

The next definition makes the term \emph{agent} precise within this model-based context, which corresponds to an internal system and an intrinsic model of the external system. The internal system state is an \emph{information state} (I-state) according to the definition introduced in \cite[Chapter 11]{LaValle-Planning-Algorithms} which is based on the von Neumann-Morgenstern notion of information for games with imperfect information.

\begin{definition}[Agent] An agent $\A$ is the tuple $\A = (\M_{U \rightarrow Y}, \ifs, U, Y, \phi, \pi)$, in which
\begin{itemize} 
    \item $\ifs$ is the set of information states. 
    \item $U$ and $Y$ are the sets of actions and observations, respectively. 
    \item Transition from an I-state $\is_{k-1}$ to the next one $\is_{k}$ with the receipt of $y_k$ is described by the information-transition function (ITF) $\phi : \ifs \times Y \rightarrow \ifs$ so that $\is_k = \phi(\is_{k-1},y_k)$.
    \item A policy $\pi : \ifs \rightarrow U$ determines an action to take at each I-state.
\end{itemize}
\end{definition}

An intrinsic model is an approximation of the external system, that is, it does not necessarily coincide with the ``reality". However, to define illusions, we need to rely on a third-person or godlike perspective. 
The reality for an agent is described as a model as well, which we call the \emph{extrinsic model}.
Let $\Omega$ be the set of all possible states of the external system, containing all the agents, from this godlike perspective, named the \emph{universe space}.

\begin{definition}[Extrinsic model] An extrinsic model for $\A\suppar{i}$ is a tuple $$\left(\Omega, \fb, U, \hb\suppar{i}, Y\suppar{i}\right)$$ consisting of: 
\begin{itemize}
    \item A set of \emph{universe states} $\Omega$. 
    \item A set of \emph{actions} $U = U^{(1)} \times U^{(2)} \times \dots \times U^{(N)}$, in which $U\suppar{n}$, $n=1,\dots, N$, is the action set of $n^{\text{th}}$ agent.
    \item A \emph{universe transition function} (UTF) $\fb : \Omega \times U \rightarrow \Omega$ such that $$\omega_{k+1} = \fb(\omega_k, u_k),$$
    in which the subscripts indicate the stage indices. 
    \item A sensor mapping $\hb\suppar{i} : \Omega \rightarrow Y\suppar{i}$ and set of observations $Y\suppar{i}$.
\end{itemize} 
\end{definition}

\subsection{Producer and receiver agents}

In this work, we consider two special agents: a \emph{producer} and a \emph{receiver}. 
A producer agent delivers a targeted perceptual experience to a receiver agent.

\begin{definition}[Receiver] A receiver is an agent $\A^r = (\M_{U^r \rightarrow Y^r}, \ifs^r, U^r, Y^r, \phi^r, \pi^r)$. An extrinsic model for the receiver is the tuple $(\Omega, \fb, U, \hb^r, Y^r)$, in which $U = U^p \times U^r$, and $U^p$ is the set of producer actions.
\end{definition}

We will consider an omniscient producer which means that its intrinsic model coincides with the extrinsic model and its sensor mapping is a bijection revealing the full state of the universe. The next definition makes this precise. 

\begin{definition}[Producer]\label{def:producer} Let $(\Omega, \fb, U, \hb^p, Y^p)$ be an extrinsic model for the producer such that $U=U^p \times U^r$, $Y^p = \Omega$ and $\hb^p$ is an identity mapping, that is, $y^p=\hb^p(\omega)=\omega$.
A producer is an agent $\A^p = (\M_{U^p \rightarrow Y^p}, \ifs^p, U^p, Y^p, \phi^p, \pi^p)$ 
such that the intrinsic model $\M_{U^p \rightarrow Y^p}=(X^p, f^p, h^p)$ satisfies that $X^p = \Omega$, $f^p = \fb$, and $h^p = \hb^p$. 
\end{definition} 

Fig.~\ref{fig:int-ext} shows these two agents coupled to the same universe. Altering the perception of a receiver implies determining actions $u^p$ which alter the universe state and causing a change in receiver I-state. A receiver behavior that is different than its intended one can be achieved by a careful selection of producer actions that result in receiver actions through receiver I-states and policy.

\subsection{Defining plausibility and illusions}
\label{sec:plausibility-illusions}

Illusions occur if the intrinsic model of the receiver does not coincide with its extrinsic model. In \cite{LaValle-Big-PE}, we have introduced a general way to model the correspondence between these models from a third-person perspective. 
Let $\A^r = (\M_{U^r \rightarrow Y^r}, \ifs^r, U^r, Y^r, \phi^r, \pi^r)$ be a receiver and let $(\Omega, \fb, U, \hb^r, Y^r)$ be its extrinsic model.

\begin{definition}[Correspondence]\label{def:correspondence} Let $C \subset X^r \times \Omega$ be called a \emph{correspondence relation}. If $(x^r, \omega) \in C$, then we say that $x^r$ corresponds to $\omega$. If $C$ is one to many and onto, then there exists a correspondence function $\alpha : \Omega \rightarrow X^r$ such that $(\alpha(w), w) \in C$ for all $\omega \in \Omega$. Thus, $x^r = \alpha(\omega)$.
\end{definition} 

\begin{figure}[t]
    \centering
    \includegraphics[width=0.98\linewidth]{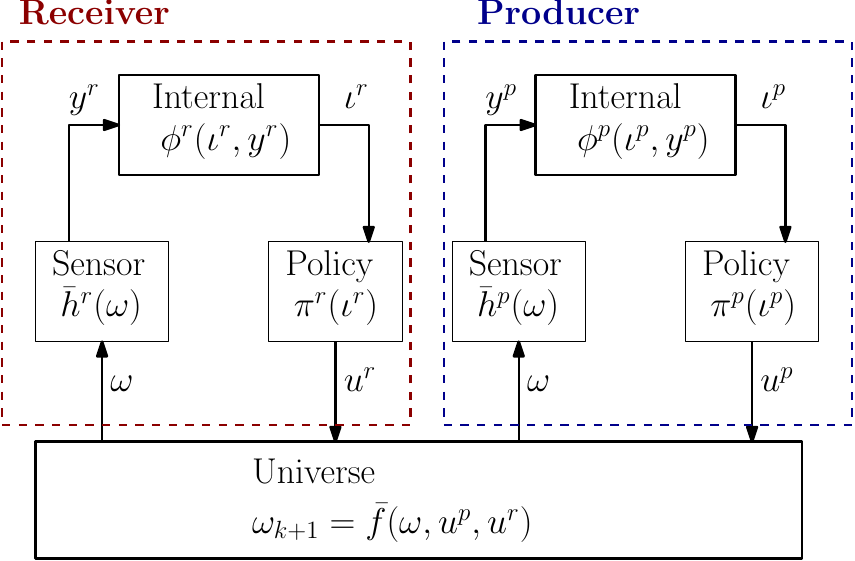}
    \caption{Producer and receiver agents coupled to the same universe.}
    \label{fig:int-ext}
\end{figure}

Because we are considering model-based internal systems, the I-state reached by interacting with the external world should be consistent with the model, that is, it should be plausible. This is captured by the next definition.

\begin{definition}[Plausibility]\label{def:plausibility} Let $M \subseteq \ifs^r \times X^r$ be a \emph{model relation}. An I-state $\is^r$ is called \emph{plausible} if there exists a $x^r \in X^r$ such that $(\is^r, x^r) \in M$. An I-state is called \emph{implausible} if it is not plausible.
\end{definition}

An illusion occurs when the receiver I-state does not correspond to the reality. This correspondence relation is defined by composing the model and correspondence relations.

\begin{definition}[Reality]\label{def:reality} Let $R \subseteq \ifs^r \times \Omega$ be a reality relation. A pair $(\is^r, \omega) \in R$ if there exists a $x^r \in X^r$ such that $(\is^r, x^r) \in M$ and $(x^r, \omega) \in C$.
\end{definition} 

Finally, the notion of illusion can be defined through the reality relation. 
\begin{definition}[Illusion]\label{def:illusion} A pair $(\is^r, \omega) \in \ifs^r \times \Omega$ is an illusion if $\is^r$ is plausible and $(\is^r, \omega) \not\in R$.
\end{definition}

To make these notions more concrete, consider the example given in Figure~\ref{fig:tower-spoofing-figure} for which $X^r=\Omega^r=\mathbb{R}^2$ and $\ifs^r = \pow(\mathbb{R}^2)$. We define the correspondence and model relations as $C = \{ (x^r,\omega) \in X^r \times \Omega^r \mid x^r=\omega\}$ and $M = \{ (\is^r,x^r) \in \ifs^r \times X^r \mid x^r \in \is^r \}$. In the example, $\is^r_k$ is a plausible I-state according to the agent's model since there exists an $x^r \in X^r$ such that $x^r \in \is^r_k$ (note that $\is^r_k$ is a singleton in the figure). However, the respective $x^r$ does not correspond to the actual state of the universe, that is, $(x^r,\omega_k)\not\in C$ since $x^r \not= w_k$. Therefore, the pair $(\is^r_k, \omega_k)$ an illusion.

\section{Optimal Control of Sensor-Induced Illusions}

In this section we introduce the problem that we address in this paper, that is, the optimal control of localization and navigation illusions for a receiver, which entails determining an optimal policy for the producer. Consequently, we describe our proposed solution considering two receiver models.

\subsection{Problem Description}\label{sec:problem_desc}

We consider a receiver moving on a plane and localizing itself using the signal intensity measurements obtained from three signal emitting towers. The producer does not have an embodiment but it can control the signal intensity of towers. 

We start with defining the extrinsic model of the receiver. The universe space is $\Omega = \R^2 \times (0,\infty)^3$. A universe state is $\omega = (\omega^r, \omega^t)$, in which $\omega^r$ is the position of the receiver with respect to a global reference frame and $\omega^t=\left(s^{(1)}, s^{(2)}, s^{(3)}\right)$ is a triple of signal intensities corresponding to the three towers at source. Let $U^r \subset \R^2$ and $U^t = \R^3$ be the set of actions corresponding to the receiver and the producer, respectively.
The universe state transitions are defined as 
\begin{equation}\label{eq:UTF}
    \omega_{k+1} = \begin{bmatrix}
        w^r_{k+1} \\ w^t_{k+1}
    \end{bmatrix} = \begin{bmatrix}
        \fb^r(\omega^r_k , u^r_k) \\ \fb^t(\omega^t_k , u^p_k)
    \end{bmatrix} = \begin{bmatrix}
        \omega^r_k + u^r_k \\ \omega^t_k + u^t_k
    \end{bmatrix},
\end{equation}
in which $u^r_k \in U^r$ and $u^t_k \in U^t$ are receiver and producer actions at stage $k$, respectively. Whereas receiver actions affect on its position, producer actions correspond to changing the signal intensity of the towers.
Consequently, $\fb = (\fb^r, \fb^t)$.
We assume that each tower has a fixed position. 
The extrinsic sensor-mapping for the receiver is 
\begin{equation}\label{eq:extr_sensor_map}
\hb^r(\omega_k)=y_k=\left(r_k^{(1)}, r_k^{(2)}, r_k^{(3)}\right),    
\end{equation}
in which $r_k^{(i)}$ is the intensity of the signal transmitted by tower $i$ measured at $\omega_k^r$ at stage $k$. 
Finally, the extrinsic model of the receiver is given as $(\Omega, \fb, U, \hb^r, Y^r)$, in which $Y^r \subseteq \R_+^3$ with $\R_+$ denoting non-negative reals and $U = U^r \times U^t$. 

The intrinsic model of the receiver is $\mathcal{M}_{U^r \rightarrow Y^r} = (X^r, f^r, h^r)$. The receiver intrinsic model assumes that the signal intensity of the towers are fixed and cannot be modified. In this case, the model corresponds to a receiver moving on a plane with $X^r = \R^2$ such that each $x^r \in X^r$ refers to the receiver position. The state transitions are defined as 
\begin{equation*}
    x^r_{k+1}=f^r(x^r_k, u^r_k)= x^r_k + u^r_k.
\end{equation*}
The intrinsic sensor-mapping assumes a fixed signal intensity at source for each tower that the sensor is calibrated to. Let $s_c=\left(s_c^{(1)}, s_c^{(2)}, s_c^{(3)}\right)$ be the vector of signal intensities used for calibration. The sensor-mapping for the intrinsic model is defined as $h^r(x_k^r)=y_k^r=\left(r_k^{(1)}, r_k^{(2)}, r_k^{(3)}\right)$ such that for $i=1,2,3$
\begin{equation}\label{eq:intensity_calibrated}
    r_k^{(i)} = \frac{s_c^{(i)}}{1+\left(\norm{t^{(i)} - x_k^r}\right)^2}.
\end{equation}

Given its intrinsic model $\mathcal{M}_{U^r \rightarrow Y^r}$, the receiver is the agent 
\begin{equation}\label{eq:receiver_model}
\A^r = (\mathcal{M}_{U^r \rightarrow Y^r}, \ifs^r, U^r, Y^r, \phi^r, \pi^r),
\end{equation}
in which, $\ifs^r \subseteq \pow(X^r)$. The ITF $\phi^r$ is given as 
\begin{equation}\label{eq:ITF_simple_rec}
     \is_{k} = \phi^r(\is^r_{k-1}, y^r_k) = (h^r)^{-1}(y^r_k),
\end{equation}
that is the state-estimate based on triangulation. Depending on the observation received, $\is_k$ is either a singleton or it is the empty set. 

The model relation $M \subseteq \ifs^r \times X^r$ is defined so that a pair $(\is^r, x^r) \in M$ if $x^r \in \is^r$. 
Therefore, an I-state $\is_k$ is \emph{plausible} if it is a singleton, that is, it is nonempty.
The following proposition specifies the conditions for an observation so that it results in a plausible I-state.
\begin{proposition}\label{prop:circles}
The preimage of $y^r_k$ under $h^r$ is a singleton if and only if $r_k^{(i)} \leq s_{c}^{(i)}$ for all $i=1,2,3$ and the intersection of the three circles centered at $t^{(i)}$ with radius $d_i = \sqrt{{s_c^{(i)}}/{r_k^{(i)}}-1}$ for $i=1,2,3$ is not empty. Otherwise, it is the empty set.
\end{proposition}
\begin{proof}(Sketch of proof)
To have $r_k^{(i)} > s_{c}^{(i)}$, $\left(\norm{t^{(i)} - x_k^r}\right)^2 < 0$ needs to be satisfied. However, there are no elements of $X^r$ that can satisfy this inequality, leading to an empty set. Given $y_k^r=\left(r_k^{(1)}, r_k^{(2)}, r_k^{(3)}\right)$, the set of $x_k^r$ that satisfy \eqref{eq:intensity_calibrated} for $r_k^{(i)}$ forms a circle. Therefore, the set of $x_k^r$ that satisfy \eqref{eq:intensity_calibrated} for $i=1,2,3$ is either unique (it is a singleton) or a solution does not exist since three circles intersect at most at one point, assuming general position. \qed
\end{proof}

The receiver's objective is to reach a goal position $x_G^r \in X^r$ and assuming that $y^r_k$ satisfies that its preimage under $h^r$ is a singleton, it follows a simple proportional control policy, that is,  
\begin{equation}\label{eq:receiver_policy}
    \pi^r(\is_k^r) = K_r (x_G^r - \is_k^r),
\end{equation}
in which $K_r$ is a scalar corresponding to controller gain.
Here, with an abuse of previously introduced notation we take $\is^r_k = \{ x^r_k\}$ as $\is^r_k = x^r_k$.

We assume that the producer is omniscient according to Definition~\ref{def:producer} meaning that its intrinsic model coincides with its extrinsic model such that $\M_{U^t \rightarrow Y^p}=(\Omega, \fb, \hb^p)$, in which $\hb^p(\omega)=y^p_k=\omega_k$ and its ITF is defined as $\is^p_k=\phi^p(\is^p_{k-1},y_k^p = \omega_k)=\omega_k$.
The producer is then expressed as the agent 
\begin{equation}\label{eq:producer_model}
    \mathcal{A}^p=(\M_{U^t \rightarrow Y^p}, \ifs^p=\Omega, U^t, Y^p=\Omega, \phi^p, \pi^p).
\end{equation}
Given this formulation, we wish to design a producer policy $\pi^p : \ifs^p \rightarrow U^t$ such that it is a solution to the following Problem~\ref{prob:prob_1}.

\begin{problem}[Regulating localization and navigation illusions]\label{prob:prob_1}
Given a receiver $\mathcal{A}^r$ as defined in \eqref{eq:receiver_model} and a producer $\mathcal{A}^p$ as defined in \eqref{eq:producer_model}, determine a policy $\pi^p$ such that for some $N$
\begin{itemize}
    \item the resulting universe state trajectory following policies $\pi^r$ and $\pi^p$ satisfies that $\omega^r_N = x_G^p$,
    \item the resulting receiver I-state trajectory $\tilde{\is}^r=(\is^r_1, \dots, \is^r_N)$ satisfies that for all $k=1,\dots, N$, $\is^r_k$ is plausible according to Definition~\ref{def:plausibility} and $\is^r_N=x_G^r$,
    \item $\pi^p$ minimizes an appropriate cost function.
\end{itemize}
\end{problem}

\subsection{Proposed Solution}\label{sec:prop_soln}

Our proposed solution relies on a producer intrinsic model that is equivalent to \eqref{eq:UTF}. To start with, for all $\is_k^r \not= \emptyset$, there exists a bijective map between the receiver I-state $\is_k^r$ and the tower signal intensities $\omega^t_k = (s^{(1)}_k, s^{(1)}_k, s^{(1)}_k)$ such that $\is_k^r \mapsto \omega^t_k$ with
\begin{equation}\label{eq:Istate_to_intensity}
s_k^{(i)} = \frac{s_c^{(i)}\left(1+\left(\norm{t^{(i)} - \omega_k^r}\right)^2\right)}{1+\left(\norm{t^{(i)} - \is_k^r}\right)^2}.
\end{equation}
Then, the system in \eqref{eq:UTF} can be rewritten as 
\begin{equation}\label{eq:affine_model_wur}
\begin{aligned}
\omega^r_{k+1} &= \omega^r_k + u^r_k \\
\iota^r_{k+1} &= \iota^r_k + u^p_k,
\end{aligned}
\end{equation}
in which $u^p_k$ is the producer action corresponding to a change in the receiver I-state and relates to $u^t_k$ similarly through \eqref{eq:Istate_to_intensity}. Let $U^p$ be the set of all possible producer actions $u_k^p$. Since $U^t$ was unbounded, $U^p$ is unbounded as well, then
the following proposition holds. 
\begin{proposition}\label{prop:control_forIstate}
There exists $u^p_k \in U^p$ for any $\is_k^r \in \ifs^r$ satisfying $\is_k^r \not= \emptyset$ such that $\is_{k+1}^r$ is plausible. 
\end{proposition}
Therefore, any plausible I-state can be generated with an appropriate selection of producer actions.
Plugging the receiver policy given in \eqref{eq:receiver_policy} into \eqref{eq:affine_model_wur} we obtain
\begin{equation}\label{eq:affine_model}
\begin{aligned}
\omega^r_{k+1} &= \omega^r_k + K_r (x_G^r - \is_k^r) \\
\iota^r_{k+1} &= \iota^r_k + u^p_k.
\end{aligned}
\end{equation}
Consequently, a change of variables by considering the error in the receiver's I-state with respect to its goal $e_k = x^r_G - \iota^r_k$ results in the following linear system
\begin{equation}\label{eq:lin_model}
\begin{aligned}
        \omega^r_{k+1} &= x^r_k + K_r e_k \\
        e_{k+1} &= e^r_k - u^p_k
\end{aligned}
\end{equation}
that can be written as 
\begin{equation}\label{eq:lin_model_AB}
\omega'_{k+1}=
\begin{bmatrix}
\omega^r_{k+1} \\ e_{k+1}
\end{bmatrix} = 
A \begin{bmatrix}
\omega^r_{k+1} \\ e_{k+1}
\end{bmatrix} + B u^p_k,
\end{equation}
in which,
\begin{equation*}
    A = \begin{bmatrix}
        I & K_rI \\
        0 & I
    \end{bmatrix}, \quad
    B = \begin{bmatrix}
        0 \\
        -I
    \end{bmatrix},
\end{equation*}
and $I$ and $0$ denote appropriately sized identity and zero matrices, respectively. The controllability matrix for \eqref{eq:lin_model_AB} is given by 
$$\text{Co} = (B \mid AB \mid A^2B \mid A^{3}B)$$ 
which has full rank for $K_r \neq 0$. Therefore, the pair $(A,B)$ is controllable for receiver policies satisfying this condition. 

Considering this model, the producer's extrinsic model is now $(\Omega', \fb', U^p, \hb', \Omega')$, in which $\Omega' = \R^2 \times (\ifs^r \setminus \emptyset)$ is the set of states and $\fb'$ is the state transition function for the system in \eqref{eq:affine_model}. We exclude pairs corresponding I-states that are empty from the state space since due to Proposition~\ref{prop:control_forIstate} they are not realized. Finally, $\hb'$ is an identity map. Considering the omniscient producer defined in Section~\ref{sec:problem_desc} under this new universe formulation, we wish to find a producer policy $\pi^p : \ifs^p \rightarrow U^p$ with $\ifs^p = \Omega'$
that solves Problem~\ref{prob:prob_1}. Because any element of $U^p$ can be uniquely mapped to an element of $U^t$, we chose not to introduce new notation and not to rewrite the problem.  

To simplify the notation used, without loss of generality, we assume that the producer goal is $\omega'^p_G=[0,0]^T$. 
The selected cost formulation
is given by
\begin{equation}\label{eq:cost_func}
    J= \frac{1}{2}\sum_{k=1}^\infty \begin{bmatrix}
\omega^r_{k} & e_{k}
\end{bmatrix} Q \begin{bmatrix}
\omega^r_{k} \\ e_{k}
\end{bmatrix} + u^{pT}_k R u^p_k,
\end{equation}
in which $Q$ and $R$ are symmetric positive-definite weight matrices determining the relative importance of each cost component. 

The producer policy $\pi^p$ that minimizes the cost given in \eqref{eq:cost_func} is 
\begin{equation}
    \pi^p(\omega'_k) = - K_p \omega'_k,
\end{equation}
with
    $K_p= (B^TPB+R)^{-1}B^TPA,$
in which $P$ is the solution to the Discrete Algebraic Riccati Equation:
\begin{equation*}
P = Q + A^T\left( P - PB(R + B^TPB)^{-1}B^TP\right)A.
\end{equation*}

Given that the pair $(A, B)$ is controllable and that the sensor mapping $\hb' $ is an identity map so that the system is observable, the closed-loop system with policy $\pi^p$ is asymptotically stable. 
Therefore, minimizing the cost in \eqref{eq:cost_func}, the resulting universe-state trajectory will converge to $(\omega_N^r = x_G^p, \is^r_N = x^r_G)$ as $N \rightarrow \infty$. Furthermore, any I-state along the resulting trajectory $\tilde{\is^r}=(\is_1, \is_2, \dots)$ is plausible by definition. Then, $\pi^p$ is an optimal solution to Problem~\ref{prob:prob_1}. 

Let $C \subset X^r \times \Omega$ be a correspondence relation (Definition~\ref{def:correspondence}) defined as $(x_r, \omega) \in C$ if  $x^r = \omega^r$. Then $(x_G^p, x^r_G)$ is an illusion if $x_G^p \not= x_G^r$ according to Definition~\ref{def:illusion}.

\subsection{Extending to advanced receivers}\label{sec:advanced_rec}
The receiver formulation in Section~\ref{sec:problem_desc} allows any producer action as long as the next I-state is a singleton. Therefore, it allows large ``jumps'' in receiver's I-state such that the distance between two consecutive position-estimates can be arbitrarily large.
In this section, we consider advanced receivers which explicitly take a motion model into account in their ITF. 

We consider nondeterministic disturbances, which come from a compact set $\Theta$ with no associated statistical information. 
The intrinsic model of the receiver $\mathcal{M}_{U^r \rightarrow Y^r} = (X^r, f^r, h^r)$ is defined similarly in Section~\ref{sec:problem_desc} with the exception that
\begin{equation}\label{eq:XTF_advanced}
    f^r(x_k^r, u_k^r) = \{ x^r_{k} + \theta \mid x^r_{k} \in X^r \land \theta \in \Theta \}, 
\end{equation}
in which $\Theta = [\theta_{min}, \theta_{max}]$ 
is a set of disturbances. 
Informally, this motion model means that given the current state $x^r_k$ of the receiver, the next state $x^r_{k+1}$ is constrained to a box centered at $x^r_k$ indicating uncertainties in the motion model and limits on receiver actions. 

Recall that $\ifs^r \subseteq \pow(X^r)$.
The intrinsic model with XTF given in \eqref{eq:XTF_advanced} can be taken into account within the receiver ITF by modifying \eqref{eq:ITF_simple_rec} such that  
\begin{equation}\label{eq:ITF_advanced}
    \is^r_k = \phi^r(\is^r_{k-1}, y^r_k) = (h^r)^{-1}(y_k^r) \cap \{ \is_{k-1}^r + \theta \mid \theta \in \Theta \},
\end{equation}
assuming $\is_{k-1}^r$ is plausible, that is, it is not empty. Similar to Section~\ref{sec:problem_desc}, we take $\is^r_k=\{x^r_k\}$ as $\is^r_k=x^r_k$ if it is not empty, and by Proposition~\ref{prop:circles} it is a singleton. If $\is_{k-1} = \emptyset$, then $\is_i = \emptyset$ for all $i > k-1$. 

Whereas it was possible to generate any plausible $\is_k^r$ in Section~\ref{sec:prop_soln} (see Proposition~\ref{prop:control_forIstate}), this is no longer possible with the receiver ITF given in \eqref{eq:ITF_advanced}. 
However, the plausibility can be taken into account by introducing limits on producer actions by imposing that the producer actions need to satisfy for all $k=1, 2 \dots$ that $ u_k^p \in \Theta$.
The resulting constrained optimal control problem can be solved in a receding horizon manner. To this end, we define the following constrained optimization problem at stage $k$, that is, \textrm{C-OPT}$(k, N)$:
\begin{equation}\label{eq:C-OPT}
\begin{aligned}
\min_{u^p_k, \dots, u^p_{k+N}} \quad & \frac{1}{2}\sum_{i=k}^{k+N} \begin{bmatrix}
\omega^r_{i+1} & e_{i+1}
\end{bmatrix} Q \begin{bmatrix}
\omega^r_{i+1} \\ e_{i+1}
\end{bmatrix} + (u^{p}_k)^T R u^p_k,\\
\textrm{s.t.} \quad & \omega'_{i+1} = A\omega'_i + Bu^p_i \quad \text{starting from $\omega'_k$}\\
  &    \theta_{\mathrm{min}} \leq u^p_i \leq \theta_{\mathrm{max}}\\ 
  & \forall i=k, \dots, k+N.
\end{aligned}
\end{equation}
Notice that \textrm{C-OPT}$(k,N)$ is a quadratic program which can be solved efficiently with the available tools. 
Now, the producer policy $\pi^p$ corresponds to an algorithm given in (Algorithm~\ref{alg:policy}). Starting from an initial state $\omega'_1$, the policy dictates solving \textrm{C-OPT}$(k,N)$ until a termination condition is reached. A termination condition could be that $\omega'_k$ is sufficiently close to the goal. 

\begin{algorithm}[t]
\caption{Producer policy}\label{alg:policy}
\begin{algorithmic}
\State $k \gets 1$
\While{!\textrm{Termination Condition}}
\State $(u_k, \dots, u_{k+N}) \gets \text{Solution to \textrm{C-OPT}$(k,N)$}$ 
\State Apply $u_k$ 
\State $k \gets k+1$
\EndWhile
\end{algorithmic}
\end{algorithm}

\section{Numerical Results}

In this section we present two numerical examples to show the effectiveness of the proposed approach considering the simple and advanced receiver formulations given in Sections~\ref{sec:problem_desc} and \ref{sec:advanced_rec}, respectively. The proposed approach is implemented in Python programming language, version 3.12. The optimization problem given in \eqref{eq:C-OPT} is solved using the tools from \texttt{cvxopt} package~\cite{andersen2021cvxopt}. The numerical examples had as termination condition the receiver being sufficiently close to its own goal, implemented with the predicate $\norm{e_k} < 0.005$.

\begin{figure}
\includegraphics[width=0.98\columnwidth]{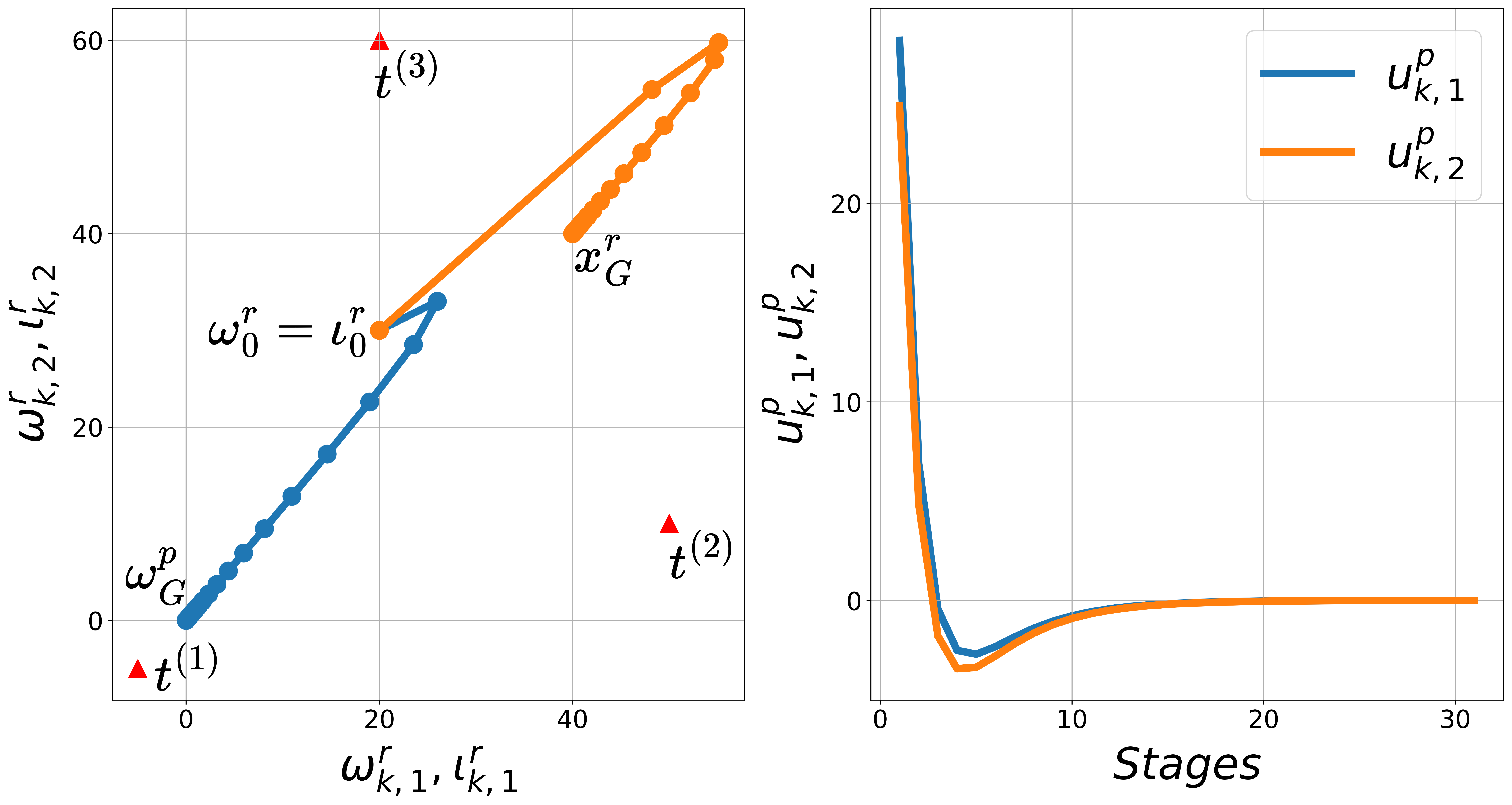}
\caption{Simulation results considering a simple receiver. (Left) Trajectories corresponding to receiver position $w^r_k$ (blue) and receiver I-state $\is_k^r$ (orange). (Right) Producer actions $u^p_k = [u^p_{k,1} \: u^p_{k,2}]^T$ as a function of stages.}
\label{fig:lqr-plots}
\end{figure}

\subsection{A simple receiver}
The first example considers a simple receiver described in Section~\ref{sec:problem_desc}. The simulation is performed considering three towers placed at positions $t^{(1)} = [{-5} \; {-5}]^T$, $t^{(2)} = [50 \; 10]^T$, $t^{(3)} = [20 \; 60]^T$. The receiver's and producer's goal are set to $x^r_G = [40 \; 40 ]^T$ and $\omega'^p_G = [0 \; 0]^T$, respectively.
The weight matrices in \eqref{eq:cost_func} are $Q = I_{4 \times 4}$ and $R = I_{2\times 2}$. The proportional gain in the receiver policy is set as $K_r = 0.3$. The optimal gain matrix in producer policy $\pi^p = -K_p \omega'_k $ is computed as 
\begin{equation*}
    K_p = \begin{bmatrix}
        -0.54 &  0 & -0.87 &  0 \\
         0 & -0.54 &  0 &         -0.87
    \end{bmatrix},
\end{equation*}
reported with each matrix entry rounded to 2 decimal digits.

Fig.~\ref{fig:lqr-plots} shows the evolution of the receiver system under producer policy $\pi^p$ starting from the initial condition $\omega^r_1 =\iota^r_1 = [20 \; 30 ]^T$ and $e_1 = x^r_G - \iota^r_0 = [20 \; 10]^T$. 

\subsection{An advanced receiver}

An advanced receiver as described in Section~\ref{sec:advanced_rec} is considered under the same setting. The only difference with respect to the previous example is that now to ensure that receiver I-states are plausible, the producer actions need to be constrained within the set $\Theta = [1\;1] \times [1\;1]$. We used $N=10$ as the number of look-ahead stages for the convex optimization \textrm{C-OPT}$(k,N)$. 

Fig.~\ref{fig:qp-plots} shows the evolution of the receiver position and its I-state. 
Also in this case, the desired receiver behavior is achieved. However, due to the introduced constraints it takes significantly longer for the producer to achieve its goal. In this example, the termination is reached at stage $64$ as opposed to the previous example for which $31$ stages were sufficient.

\begin{figure}
\centering
\includegraphics[width=\columnwidth]{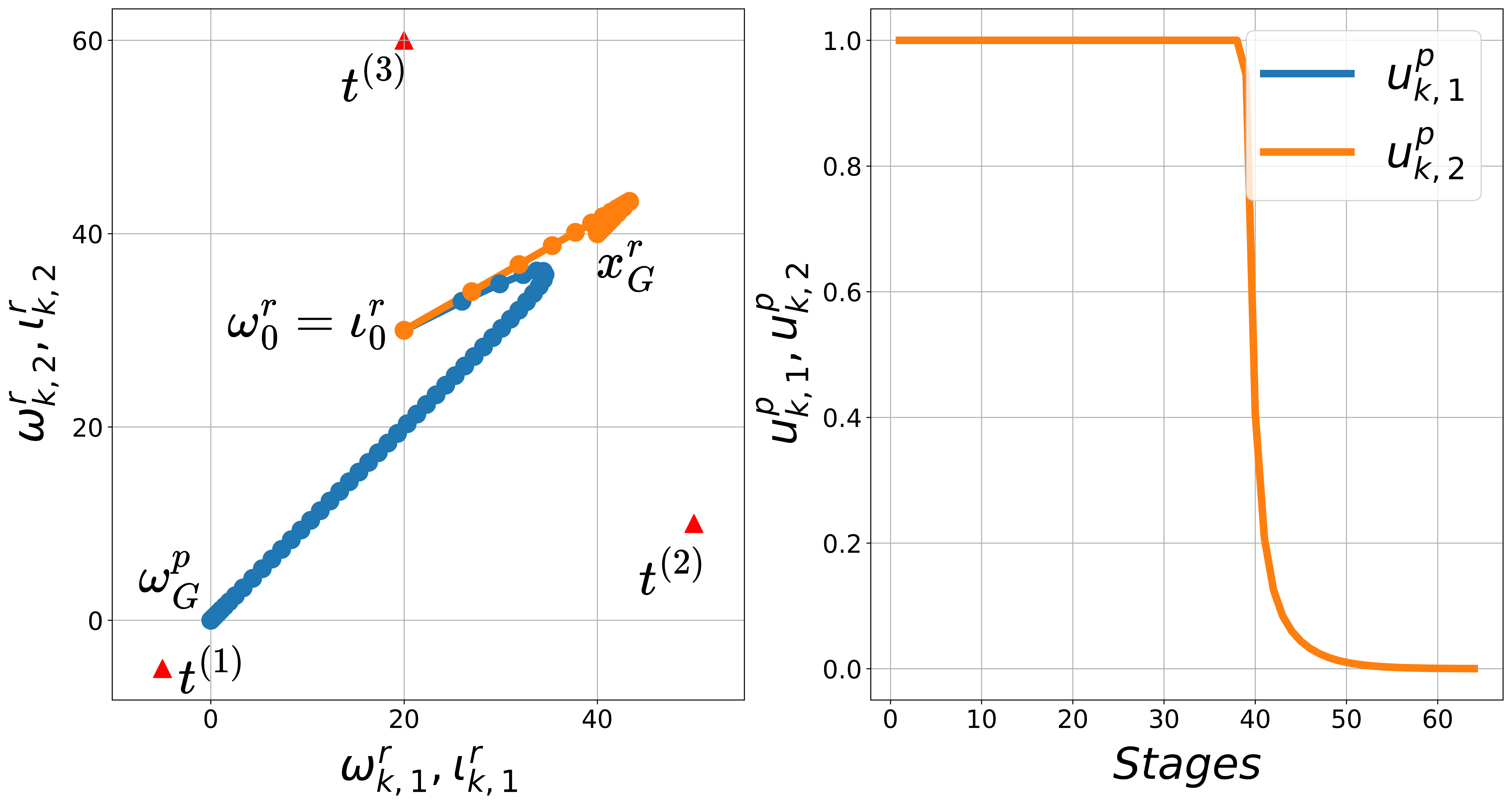}
\label{fig:constrained-qp-sim}
\caption{Simulation results considering an advanced receiver. (Left) Trajectories corresponding to receiver position $w^r_k$ (blue) and receiver I-state $\is_k^r$ (orange). (Right) Producer actions $u^p_k = [u^p_{k,1} \: u^p_{k,2}]^T$ as a function of stages.}
\label{fig:qp-plots}
\end{figure}

\section{CONCLUSION}
In this work, we have considered designing a producer that is able to regulate the perceptual experience of a receiver agent so that the latter's actions achieve the producer's goal rather than its own. The novel mathematical formulation of the interaction between a receiver and producers as a control system, and posing the problem of regulating the receiver perceptual experience (its I-state) as a control problem enabled the use of standard methods.
A first extension of the proposed models comes in the form of trajectory tracking: rather than driving the state of the system to the origin, the objective can be to have it follow a reference trajectory $(\omega^{\circ}_1, \omega^{\circ}_2, \dots, \omega^{\circ}_{N+1})$ with as little error as possible. This can be accomplished by a simple modification in the cost function \eqref{eq:cost_func}.

Our formulation is general and allows considering different receiver and producer agents. Future work includes considering more advanced receivers that potentially take into account more accurate motion models or even potential producers within their intrinsic model. Therefore, further constraining the plausible set of I-states and consequently, producer actions. Another direction is to take sensing uncertainties into account in the receiver model, in this case, the preimage of an observation coming from a tower is an annulus instead of a circle. 
Extensions to probabilistic I-states is also possible, and already considered in \cite{LaValle-Big-PE}. However, probabilistic models introduce additional complications in terms of determining what is plausible especially if the I-states correspond to a probability distribution over an unbounded set resulting in any one of those to be plausible. In this case, introducing thresholds or considering distributions with bounded support can be explored. Some progress along this direction was made in \cite{Medici-Thesis}.

Another interesting direction is to consider weaker producers, that is, producers that are not omniscient. In this case, the intrinsic model of the producer differs from its extrinsic model, leading to a partially observable system. Therefore, determining a producer policy with limited information and analyzing the stability of the closed-loop system opens up interesting areas to explore.

\bibliographystyle{IEEEtran}
\bibliography{root}

\end{document}